\documentclass[preprint,12pt]{elsarticle}

\usepackage{hyperref}
\usepackage[edges]{forest} 
\usepackage{bm}
\usepackage{mathtools}
\usepackage{algorithm}
\usepackage{algorithmic}
\DeclareMathOperator*{\argmax}{arg\,max}
\usepackage{booktabs}
\usepackage{array}
\usepackage{adjustbox} 
\usepackage{multirow}
\usepackage{soul}
\usepackage{url}

\usepackage{amssymb}
\usepackage{amsthm}

\journal{ }

\begin{document}

\begin{frontmatter}



\title{Significance-Based Categorical Data Clustering}


\author[a]{Lianyu Hu} 
\ead{hly4ml@gmail.com}

\author[a]{Mudi Jiang}

\author[a]{Yan Liu}

\author[a,b]{Zengyou He} 
\ead{zyhe@dlut.edu.cn}

\affiliation[a]{organization={School of Software, Dalian University of Technology}, 
	city={Dalian, 116024}, 
	country={China}}

\affiliation[b]{
	organization={Key Laboratory for Ubiquitous Network and Service Software of Liaoning Province},
	city={Dalian, 116024}, 
	country={China}}

\begin{abstract}
Although numerous algorithms have been proposed to solve the categorical data clustering problem, how to access the statistical significance of a set of categorical clusters remains unaddressed. To fulfill this void, we employ the likelihood ratio test to derive a test statistic that can serve as a significance-based objective function in categorical data clustering.  Consequently, a new clustering algorithm is proposed in which the significance-based objective function is optimized via a Monte Carlo search procedure. As a by-product, we can further calculate an empirical $p$-value to assess the statistical significance of a set of clusters and develop an improved gap statistic for estimating the cluster number. Extensive experimental studies suggest that our method is able to achieve comparable performance to state-of-the-art categorical data clustering algorithms. Moreover, the effectiveness of such a significance-based formulation on statistical cluster validation and cluster number estimation is demonstrated through comprehensive empirical results. 
\end{abstract}

\begin{highlights}
\item The categorical data clustering issue is studied from a significance testing aspect. 
\item A new objective function is derived for clustering categorical data. 
\item A significance-based algorithm is proposed for clustering categorical data. 
\item An empirical $p$-value is calculated for clustering tendency evaluation.
\end{highlights}

\begin{keyword}
Clustering \sep Categorical Data \sep Statistical Significance \sep Likelihood Ratio Test

\end{keyword}

\end{frontmatter}


\section{Introduction}
\label{sec:introduction} 
In many real applications \cite{Andreopoulos2009, Powers2008}, we need to conduct cluster analysis on categorical data sets. In such types of data sets, each feature is qualitative in nature and takes discrete values. To solve the categorical data clustering problem, many effective algorithms have been proposed during the past decades \cite{Naouali2020}. Although existing categorical data clustering algorithms are developed based on different principles \cite{Bai2015}, most of them typically formulate the clustering problem as an optimization problem. Such optimization-based approaches can always report a set of clusters, no matter whether really there exists a cluster structure \cite{Adolfsson2019, Hennig2015}. Thus, we need to develop a notion of statistical significance for assessing the clustering result on a categorical data set. 

To date, there are already some research efforts on how to assess the statistical significance of clusters (e.g., \cite{Ling1973, Liu2008, Petrie2013, Bock1985, Smith1984, Maitra2012}). In general, these methods fall into two categories: testing the statistical significance of an individual cluster (e.g., \cite{Ling1973, Petrie2013}) and assessing the statistical significance of an entire partition (e.g., \cite{Liu2008, Maitra2012}). These significance-based evaluation methods can be either used for validating clustering results generated by existing algorithms or employed as an objective function for directly finding statistically significant clusters. To obtain statistically significant clusters, one can either sequentially extract individual clusters from the given data set or directly divide the data set into $k$ clusters ($k>1$).

Unfortunately, almost all the existing significance-based methods are developed for assessing and detecting clusters from numerical data sets. One notable exception is the clustering algorithm introduced in \cite{Zhang2006}, which iteratively finds a statistically significant cluster from the categorical data set. That is, the method in \cite{Zhang2006} is developed for the purpose of assessing the statistical significance of an individual categorical cluster. To our knowledge, there is still no method available in the literature that can be directly used for quantifying the statistical significance of a set of non-overlapped categorical clusters (i.e., a partition of the given categorical data set). 

In this paper, we investigate the issue of assessing the statistical significance of a set of categorical clusters.  Under the null hypothesis of an individual cluster (the categorical data set is a randomized data set with no clustering structure) and the alternative hypothesis of $k$ clusters, we obtain a test statistic based on the likelihood ratio (LR) test. To demonstrate the effectiveness of such a testing-based approach, we first develop a new clustering algorithm in which the test statistic is utilized as the objective function. Consequently, we show that it is feasible to assess whether a set of detected clusters is statistically significant based on an empirical $p$-value derived from the null distribution of test statistics. Finally, the test statistic can also be used for determining the number of true clusters when it is plugged into the method based on Gap statistic \cite{Tibshirani2001}. Experimental results show that the significance-based clustering algorithm can beat existing algorithms.  Meanwhile, the test statistic and the derived $p$-value are effective for evaluating the realness of clustering results and choosing a proper number of clusters. 

The main contributions of this paper can be summarized as follows:
\begin{itemize}
	\item A new significance-based objective function is proposed, which is the first attempt to assess the statistical significance of a set of categorical clusters.
	\item Based on significance-based objective function, a new clustering algorithm based on the Monte Carlo search procedure is presented. Experimental results on real categorical data sets show that the proposed method is competitive with existing categorical data clustering methods in terms of both accuracy and running time.
	\item We calculate an empirical $p$-value to determine whether a partition over a categorical data set is statistically significant. We also discuss how to choose the number of clusters using the LR test statistic. Empirical results on both simulated and real data sets verify the effectiveness of our method on clustering tendency evaluation and model selection. 
\end{itemize}

The remaining parts of this paper are organized as follows: Section \ref{sec:relatedwork} discusses related work. Section \ref{sec:methods} presents our method in detail. Section \ref{sec:experiments} evaluates the proposed method through experiments on real and simulated data sets. Section \ref{sec:conclusion} concludes this paper and discusses future works.

\section{Related Work}
\label{sec:relatedwork}
Since we focus on assessing the statistical significance of a set of categorical clusters, this section mainly discusses related algorithms for clustering categorical data and assessing the statistical significance of clusters.

\subsection{Categorical data clustering}
Generally, each categorical data clustering method comprises two key components: an objective function and a search procedure. Given an objective function, we typically conduct the search procedure in a partitional (e.g., $k$-modes \cite{Huang1998}) or hierarchical (e.g., ROCK \cite{Guha2000}) manner. Current objective functions in categorical data clustering can be divided into the following two types.

(1) \textit{Pairwise comparison-based objective function} is defined according to the pairwise distance between two data points. The Hamming distance is probably the most widely used distance metric for categorical data \cite{Ng2007, Bera2021}, and more distance measures are summarized and discussed in \cite{Boriah2008, Zhang2019}. Recently, similarity-based representation learning methods have received much attention \cite{IamOn2010, Jian2018, Zheng2020, Zhang2022, Zhu2022, Bai2022}.

(2) \textit{Group-based objective function} evaluates a cluster or a set of clusters without relying on a specific distance metric. Entropy is a classical notion for measuring the uncertainty of discrete random variables \cite{Thomas2006} and is widely used in categorical data clustering \cite{Chen2008, Sharma2020}. It is natural to employ the entropy to evaluate the goodness of a set of clusters \cite{Barbara2002} or measure the similarity between clusters \cite{Bai2018}. In hierarchical categorical clustering, entropy is used to measure the compactness (homogeneity) of a cluster \cite{Andritsos2004, Cesario2007, Xiong2012}.

In this paper, we propose an objective function that can be regarded as a group-based one which has a clear statistical interpretation. This new objective function can address what is a natural clustering of categorical data sets from a significance testing aspect. Those entropy-based objective functions \cite{Li2004} can be derived from probabilistic clustering models, such as the Bernoulli mixture models. In contrast, our objective function is obtained from a hypothesis testing procedure.

\subsection{Statistical significance of clusters}
To assess the statistical significance of clusters, many effective methods have been proposed. These methods are either generic or developed for specific clustering algorithms (e.g., $k$-means \cite{Liu2008}, Ward’s linkage clustering \cite{Kimes2017}, density-based clustering \cite{Xie2021}). In general, these methods can be classified into the following two categories. 

(1) The evaluation of an \textit{individual cluster}. When testing an individual cluster, there are usually two cases. One is to test the uniformity for a complete data set, which is recommended to be the first step in data analysis \cite{Petrie2013}. The other is to test the homogeneous subsets of a data set in which within-cluster variance can be used \cite{Park2009}.

(2) The evaluation of an \textit{entire partition} (a set of non-overlapped clusters). These methods mainly intend to determine whether a given clustering result in terms of a partition is statistically significant. Moreover, some methods are devoted to choosing a proper number of clusters by comparing two partitions under different cluster numbers \cite{Maitra2012}. 

As summarized in Fig.\ref{fig:sig-roadmap}, existing significance-based methods assess the statistical significance using either an analytical method or a sampling method. Although the analytical method is appealing since it can provide a closed-form solution, it is a challenging task to derive the analytical solution due to the lack of a proper null model. As a result, existing methods mainly employ sampling methods such as using Monte Carlo \cite{Xie2022} and Bootstrap \cite{Helgeson2021} to obtain an empirical significance evaluation result.

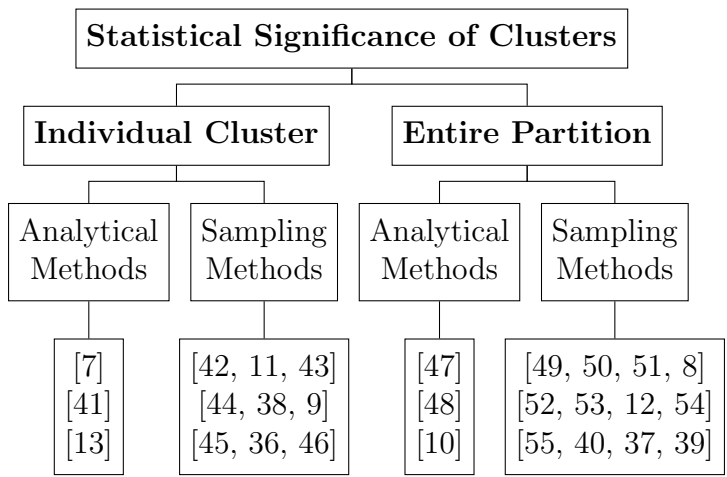
\begin{figure}[t]
	\centering
	\begin{forest}
		for tree={
			draw,
			align=center
		},
		forked edges,
		[\textbf{Statistical Significance of Clusters}
		[\textbf{Individual Cluster}
		[Analytical\\Methods
		[\cite{Ling1973}\\
		\cite{BaileyJr1982}\\
		\cite{Zhang2006}
		]
		]
		[Sampling\\Methods
		[\cite{Lee1979, Smith1984, Suzuki2006}\\
		\cite{Wieland2007, Park2009, Petrie2013}\\
		\cite{Bruzzese2015, Kimes2017, Liu2019}
		]
		]
		]
		[\textbf{Entire Partition}
		[Analytical\\Methods
		[\cite{Ling1976}\\
		\cite{Hartigan1978}\\
		\cite{Bock1985}
		]
		]
		[Sampling\\Methods
		[\cite{Gordon1996, Auffermann2002,Pacifico2007,Liu2008} \\
		\cite{Fuentes2009, Neill2012, Maitra2012, Huang2015}\\
		\cite{Valk2021, Helgeson2021, Xie2021, Xie2022}
		]
		]
		]
		]
	\end{forest}
	\caption{The roadmap of existing methods for measuring the statistical significance of clusters.}
	\label{fig:sig-roadmap}
\end{figure}

Except for the method in \cite{Zhang2006}, all the methods in Fig.\ref{fig:sig-roadmap} are not developed for categorical data sets. The algorithm in \cite{Zhang2006} sequentially extracts statistically significant clusters from a categorical data set. That is, in each iteration, one statistically significant cluster is identified and removed from the remaining data. To assess the statistical significance of an individual cluster, the chi-square test is employed for finding a cluster center and determining the cluster radius.  

In this paper, we investigate the issue of assessing the statistical significance of a partition of the categorical data, which is different from the algorithmic issue of assessing an individual cluster in \cite{Zhang2006}.

\section{Methods}
\label{sec:methods}
\subsection{Problem formulation}
Let $X=[X_1,\cdots,X_M]$ denote a categorical data set with $M$ attributes. Each $X_m=[x_{m1},\cdots,x_{mN}]$ is a sequence of $N$ categorical values  for $1\leq m\leq M$, where $N$ is the number of objects. We use $x_m$ to denote a categorical variable for $X_m$ that takes on values within the set $\{1,\cdots,Q_m\}$. Given $X_m$, we use $N_{mq}$ to represent the frequency of $x_m=q$ for $1\leq q\leq Q_m$. 

The aim of cluster analysis is to divide $X$ into $K$ non-overlapped clusters  $\pi=\{\pi_1,\cdots,\pi_K\}$, where the $k$-th cluster contains $N^{(k)}$ objects for $1\leq k\leq K$. $X_m^{(k)}=[x_{m1}^{(k)},\cdots,x_{mN^{(k)}}^{(k)}]$ is used to denote the sequence of categorical values of the $m$-th attribute in the $k$-th cluster and $N_{mq}^{(k)}$ is used to represent the frequency of $x_m^{(k)}=q$ for $1\leq q\leq Q_m$. To evaluate the goodness of a given partition $\pi$, we typically define an objective function that can output a numeric score. In this paper, we focus on evaluating the statistical significance of an entire partition. That is, the test statistic and the corresponding $p$-value will be employed as the objective function. 

\subsection{Significance-based objective function}
To test the statistical significance of a set of categorical clusters, we first present a naive random model that can describe the underlying categorical data-generation mechanism under the null hypothesis of an individual cluster. Then, we obtain a test statistic using the LR test under the alternative hypothesis that there are $K$ clusters in the given data set. The test statistic can serve as the significance-based objective function for testing the hypothesis: $H_0$: There is no clustering structure versus $H_1$: There are $K$ clusters. 

\subsubsection{Null and alternative models}
A categorical data set $X$ is generated from $M$ trials, assuming \footnote{Although this assumption is a simplification of real-world data in which attributes may not be independent, some analyses have shown that this independence assumption is theoretical effective (e.g., in Naive Bayes classifiers \cite{Zhang2004}).} that each attribute $X_m$ is independent of the other attributes. For each attribute, $X_m$ is assumed to be generated from a multinomial distribution \cite{Murphy2023}. 

Let $x_m \sim \mu(\bm{\theta}_{m},Q_m)$ be a discrete random variable drawn from a multinomial distribution with $Q_m$ potential discrete outcomes, where $\bm{\theta}_{m}=[\theta_{m1},$ $\cdots,\theta_{mQ_m}]$ is the parameter vector. Let $\bm{\theta_0}=[\theta_{11},\cdots,\theta_{MQ_M}]$ denote the parameter vector of length  $\sum\limits_{m=1}^{M} Q_m$ under the null hypothesis, where each $\theta_{mq}$ ($m=1,\cdots,M$, $q=1,\cdots,Q_m$) denotes the expected probability of the $q$-th categorical value in $X_m$ from $X$. As we have a vector of sufficient statistics $N_{mq}$ ($m=1,\cdots,M$, $q=1,\cdots,Q_m$) for the parameter vector $\bm{\theta_0}$, the multinomial likelihood function for $\bm{\theta}_m$ given $X_m$ is:
\begin{equation}
	L(\bm{\theta}_m|X_m)=\prod\limits_{n=1}^{N} \mu (x_{mn}|\bm{\theta}_m)=\prod\limits_{q=1}^{Q_m}{\left(\theta_{mq}\right)^{N_{mq}}}.
\end{equation}

Then, given $X$, the likelihood function for $\bm{\theta_0}$ is:
\begin{equation}
	L (\bm{\theta_0}|X)=\prod\limits_{m=1}^{M} L(\bm{\theta}_m|X_m)=\prod\limits_{m=1}^{M}\prod\limits_{q=1}^{Q_m}{(\theta_{mq})^{N_{mq}}}.
\end{equation}

\newtheorem{thm}{Theorem}
\newtheorem{lem}[thm]{Lemma}
\begin{lem}
	\label{lem1}
	The maximum likelihood estimators (MLE)
	for unknown parameters are $\hat{\theta}_{mq}=\frac{N_{mq}}{N}$.
\end{lem}
\begin{proof}
	The MLE on the log-likelihood is equivalent to its counterpart on the original likelihood function since logarithm is strictly concave. The logarithmic $L(\bm{\theta_0}|X)$ can be written as:
	\begin{equation}
		\log L(\bm{\theta_0})=\sum\limits_{m=1}^{M}\sum\limits_{q=1}^{Q_m}N_{mq}\log{\theta_{mq}}.
	\end{equation}
	Given the constraint that $\sum\limits_{m=1}^{M}\sum\limits_{q=1}^{Q_m}{\theta_{mq}}=\sum\limits_{m=1}^{M}1=M$, we have the Lagrangian as:
	\begin{equation}
		\mathcal{L}(\bm{\theta_0},\lambda)=\log L(\bm{\theta_0})+\lambda(M-\sum\limits_{m=1}^{M}\sum\limits_{q=1}^{Q_m}{\theta_{mq}}).
	\end{equation}
	We differentiate the Lagrangian w.r.t. $\theta_{mq}$
	\begin{equation}
		\begin{split}
			\frac{\partial}{\partial\theta_{mq}} \mathcal{L}(\bm{\theta_0}, \lambda) & = \frac{\partial}{\partial\theta_{mq}}\log L(\bm{\theta_0})
			 + \frac{\partial}{\partial\theta_{mq}} \lambda (M-\sum\limits_{m=1}^{M}\sum\limits_{q=1}^{Q_m}{\theta_{mq}})\\ & = \frac{\partial}{\partial\theta_{mq}}\sum\limits_{m=1}^{M}\sum\limits_{q=1}^{Q_m}N_{mq}\log{\theta_{mq}} + \frac{\partial}{\partial\theta_{mq}}\lambda M - \frac{\partial}{\partial\theta_{mq}}\sum\limits_{m=1}^{M}\sum\limits_{q=1}^{Q_m}\lambda {\theta_{mq}}  \\ & = \sum\limits_{m=1}^{M}\sum\limits_{q=1}^{Q_m}\frac{\partial}{\partial\theta_{mq}}(N_{mq}\log{\theta_{mq}}-\lambda \theta_{mq}) = \sum\limits_{m=1}^{M}\sum\limits_{q=1}^{Q_m}(\frac{N_{mq}}{\theta_{mq}}-\lambda). 
		\end{split}
	\end{equation}
	Setting $\frac{\partial}{\partial\theta_{mq}} \mathcal{L}(\bm{\theta_0}, \lambda)=0$ for each $(m,q)$ yields
	\begin{equation}
		\label{eq:MLE}
		\theta_{mq}=\frac{N_{mq}}{\lambda}.
	\end{equation}
	Using the sum-to-$M$ constraint and $\sum\limits_{q=1}^{Q_m}N_{mq}=N$, we have
	\begin{equation}
			M =\sum\limits_{m=1}^{M}\sum\limits_{q=1}^{Q_m}{\theta_{mq}} = \frac{1}{\lambda}\sum\limits_{m=1}^{M}\sum\limits_{q=1}^{Q_m}N_{mq}\\
			=\frac{1}{\lambda}\sum\limits_{m=1}^{M}N=\frac{1}{\lambda}(M\cdot N).
	\end{equation}
	
	Hence we have $\lambda=N$ and bring it into Equation \ref{eq:MLE}, the MLE for each parameter can be written as $\hat{\theta}_{mq}=\frac{N_{mq}}{N}$.
\end{proof}

Given a partition $\pi$ with $K$ non-overlapped clusters, $\pi_k$ is used to denote the $k$-th cluster. A categorical data set $X$ under a partition $\pi$ is generated from $K$ trials, assuming that objects in each cluster $\pi_k$ is independent of the objects in other clusters. Let $\bm{\theta_1}=\left[\theta^{(1)}_{11},\cdots,\theta^{(K)}_{MQ_M}\right]$ denote the parameter vector under the alternative hypothesis of $K$ clusters, where $\theta^{(k)}_{mq}$ ($q=1,\cdots,Q_m$, $k=1,\cdots,K$) denotes the expected probability of the $q$-th categorical value in $X_m^{(k)}$ conditioned on $\pi_k$. In each cluster, we have a vector of sufficient statistics $N_{mq}^{(k)}$ ($m=1,\cdots,M$) for the parameter vector $\bm{\theta_1}$, and the likelihood function for $\bm{\theta_1}$ given $\pi$ is:
\begin{equation}
		L(\bm{\theta_1}|X,\pi)  =\prod\limits_{k=1}^{K}L\left(\bm{\theta_1}|\pi_k\right)  \\
		= \prod\limits_{k=1}^{K}\prod\limits_{m=1}^{M}\prod\limits_{q=1}^{Q_m}{\left(\theta_{mq}^{(k)}\right)^{N_{mq}^{(k)}}},
\end{equation}
where $\prod\limits_{q=1}^{Q_m}{\theta_{mq}^{(k)}}=1$, $\sum\limits_{q=1}^{Q_m}N_{mq}^{(k)}=N^{(k)}$.

The MLE for each $\hat{\theta}_{mq}^{(k)}$ is $\frac{N_{mq}^{(k)}}{N^{(k)}}$, which can be derived in a similar manner as shown in the previous proof.

\subsubsection{LR test}
The LR test is a natural choice for quantifying the significance of a partition under two models, especially where these two models are nested. In our context, the LR test compares the goodness-of-fit between a more complex model $L^*(\boldsymbol{\theta_1}|X,\pi^*)$ and a less complex model $L^*(\boldsymbol{\theta_0}|X)$ in terms of the number of parameters, where $L^*(\boldsymbol{\theta_0}|X)$  and $L^*(\boldsymbol{\theta_1}|X,\pi^*)$ are the maximal likelihoods under the null and alternative hypotheses, respectively. The likelihood ratio test statistic $\Lambda^*$ has the following form:
\begin{equation}
		\Lambda^*(\boldsymbol{\theta}|X,\pi^*) = \frac{L^*(\boldsymbol{\theta_0}|X)}{L^*(\boldsymbol{\theta_1}|X,\pi^*)}  =\frac{\prod\limits_{m=1}^{M}\prod\limits_{q=1}^{Q_m}{(\theta_{mq})^{N_{mq}}}}{\prod\limits_{k=1}^{K}\prod\limits_{m=1}^{M}\prod\limits_{q=1}^{Q_m}{\left(\theta_{mq}^{(k)}\right)^{N_{mq}^{(k)}}}}.
\end{equation}

When $N\to \infty$ and under certain regularity conditions \cite{Wilks1938}, the probability distribution of $-2 \ln \Lambda^*$ approaches the chi-square distribution with $\sum\limits_{k=1}^{K}\sum\limits_{m=1}^{M}Q_m-\sum\limits_{m=1}^{M} Q_m$ degrees of freedom.

Hence, we obtain a significance-based objective function named \textbf{S}implified \textbf{R}atio \textbf{S}tatistic (SRS) that is derived from $-2 \ln \Lambda^*$. The general clustering problem is to find a set of $K$ clusters $\pi^*$ from the partition space  where $\pi^*$ is the maximizer of $-2 \ln \Lambda^*$.

Given a categorical data set $X$, we simplify $-2 \ln \Lambda^*$ as follows:
\begin{equation}
	\label{eq: RS}
	\begin{split}
		-2 \ln \Lambda^* & = -2\ln \frac{L^*(\boldsymbol{\theta_0}|X)}{L^*(\boldsymbol{\theta_1}|X,\pi^*(K))} = -2\ln \frac{\prod\limits_{m=1}^{M}\prod\limits_{q=1}^{Q_m}{(\theta_{mq})^{N_{mq}}}}{\prod\limits_{k=1}^{K}\prod\limits_{m=1}^{M}\prod\limits_{q=1}^{Q_m}{\left(\theta_{mq}^{(k)}\right)^{N_{mq}^{(k)}}}}\\
		&= -2\big[\ln \prod\limits_{m=1}^{M}\prod\limits_{q=1}^{Q_m}{(\theta_{mq})^{N_{mq}}} -\ln \prod\limits_{k=1}^{K}\prod\limits_{m=1}^{M}\prod\limits_{q=1}^{Q_m}{\left(\theta_{mq}^{(k)}\right)^{N_{mq}^{(k)}}}\big].
	\end{split}
\end{equation}

As $\big[\ln \prod\limits_{m=1}^{M}\prod\limits_{q=1}^{Q_m}{(\theta_{mq})^{N_{mq}}}\big]$ is a constant for a given data set $X$, the objective function for a given partition of $K$ clusters can be written as:
\begin{equation}
	\label{eq: SRS}
	\begin{split}
		\text{SRS}(X,\pi) & = -\ln \prod\limits_{k=1}^{K}\prod\limits_{m=1}^{M}\prod\limits_{q=1}^{Q_m}{\left(\theta_{mq}^{(k)}\right)^{N_{mq}^{(k)}}} = -\sum\limits_{k=1}^{K}\sum\limits_{m=1}^{M}\sum\limits_{q=1}^{Q_m}\ln\left(\frac{N^{(k)}_{mq}}{N^{(k)}}\right)^{N_{mq}^{(k)}}\\
		&= -\sum\limits_{k=1}^{K}\sum\limits_{m=1}^{M}\sum\limits_{q=1}^{Q_m}\big({N_{mq}^{(k)}}\cdot\ln {N_{mq}^{(k)}}  -{N_{mq}^{(k)}}\cdot\ln {N^{(k)}}\big) \\
		&= \sum\limits_{k=1}^{K}\sum\limits_{m=1}^{M}\sum\limits_{q=1}^{Q_m}{N_{mq}^{(k)}}\cdot\ln {N^{(k)}} - \sum\limits_{k=1}^{K}\sum\limits_{m=1}^{M}\sum\limits_{q=1}^{Q_m}{N_{mq}^{(k)}}\cdot\ln {N_{mq}^{(k)}}\\
		&= M\cdot\sum\limits_{k=1}^{K}\big(N^{(k)}\cdot \ln {N^{(k)}}\big) - \sum\limits_{m=1}^{M}\sum\limits_{k=1}^{K}\sum\limits_{q=1}^{Q_m}\big({N_{mq}^{(k)}}\cdot \ln {N_{mq}^{(k)}}\big).
	\end{split}
\end{equation}

To calculate the objective function according to Equation \ref{eq: SRS}, we need to know the number of attributes, the number of objects in each cluster and the frequency of each attribute value in each cluster. Obviously, these values are easy to obtain and the objective function can be calculated quickly in an incremental manner. Moreover, we have the following remarks on this objective function when it is employed for evaluating a partition of categorical clusters.

\begin{itemize}
	\item To obtain the optimal $\pi^*$ for a given $X$, we try to minimize SRS ($\geq 0$) since it is derived from Equation \ref{eq: RS}, where maximizing $(-2 \ln \Lambda)$ is equivalent to minimizing SRS. The objective function can measure the total homogeneity of a set of categorical clusters. The smaller the SRS value, the more likely it is to obtain a set of compact categorical clusters. For each attribute in each cluster, we have $\sum\limits_{q=1}^{Q_m}{N_{mq}^{(k)}}\cdot \ln {N_{mq}^{(k)}}\leq \sum\limits_{q=1}^{Q_m}{N_{mq}^{(k)}}\cdot \ln {N^{(k)}} = (\sum\limits_{q=1}^{Q_m}{N_{mq}^{(k)}})\cdot \ln {N^{(k)}}=N^{(k)}\cdot \ln N^{(k)}$. If there are more identical attribute values, $\sum\limits_{q=1}^{Q_m}{N_{mq}^{(k)}}\cdot \ln {N_{mq}^{(k)}}$ tends to be larger which means that the second term in the SRS function will be smaller. 
	\item The SRS shares some similarities with entropy-based objective functions such as \textit{expected entropy} (EE) \cite{Li2004, Barbara2002}. As shown in the \ref{sec:appendixA}, the upper and lower bounds of EE can be expressed in terms of SRS.
\end{itemize}

\subsection{The clustering algorithm}
\label{sec: clustering_algorithm}
To verify the effectiveness of the proposed significance-based objective function, we develop a new clustering algorithm in which SRS is utilized as the objective function. As shown in Algorithm \ref{alg: sigcat}, the new clustering algorithm is named $K$-SigCat (\textbf{Sig}nificance-based \textbf{Cat}egorical data clustering with $K$ clusters), whose input is a categorical data set and a user-specified parameter $K$. The parameter $K$ is used to specify the number of clusters and the algorithm returns a locally optimal partition with respect to SRS.

\begin{algorithm}[h]
	\caption{$K$-SigCat}
	\label{alg: sigcat}
	\begin{algorithmic}[1]
		\REQUIRE A categorical data set $X$, and the number of \\
		clusters $K$.
		\ENSURE A set of $K$ locally optimal clusters $\pi^*$.
		\STATE  \textbf{Set:} $no\_change = 0$.
		\STATE \textbf{Initialization:} Put all objects into one cluster:\\$\pi^*=\{\{X\}_1,\{\}_2,\cdots,\{\}_K\}$, \label{algline: int}
		\STATE $eva^*=\text{SRS}(X,\pi^*)$. \label{algline: endint}
		\REPEAT  \label{algline: rpt}
		\STATE $[\pi, eva]=\text{Local\_Search}(X, \pi^*, eva^*)$
		\IF{$eva< eva^*$} \label{algline: if}
		\STATE $\pi^*\gets\pi$ 
		\STATE $eva^*= eva$
		\STATE $no\_change= 0$
		\ELSE
		\STATE $no\_change++$ \label{algline: noc}
		\ENDIF \label{algline: endif}
		\UNTIL{$no\_change == N(K-1)$} \label{algline: endrpt} 
		\RETURN $\pi^*$
	\end{algorithmic}
\end{algorithm}

$K$-SigCat is to find a locally optimal partition by continuously minimizing SRS on the current partition  $\pi^*$. We use a Monte Carlo method to iteratively reduce the SRS value. Initially, we put all objects into one cluster and set remaining $K-1$ clusters to be empty. Then, we use a local search produce to repeatedly refine the partition. This local search procedure will be terminated until the consecutive failure times on reducing the SRS value exceeds $N(K-1)$. 

\begin{algorithm}[h]
	\caption{Local\_Search}
	\label{alg: local_search}
	\begin{algorithmic}[1]
		\REQUIRE A categorical data set $X$, a partition $\pi$,\\
		and its SRS value $eva$.
		\ENSURE  A partition $\pi^{\prime}$, and its SRS value $eva^{\prime}$.
		\STATE Randomly select an object $x$ in a cluster $\pi_A$
		\STATE Randomly put $x$ into another cluster $\pi_B$
		\STATE  $N^{(A)}\gets N^{(A)}-1$ and $N^{(B)}\gets N^{(B)}+1$  \label{algline: N}
		\STATE  $N^{(A)}_{mq}\gets N^{(A)}_{mq}-1$ and $N^{(B)}_{mq}\gets N^{(B)}_{mq}+1$ when $q=x_m$ \label{algline: Nmq}
		\STATE Update partition $\pi^{\prime}\gets \pi$:\\
		$\pi^{\prime}_A\gets \pi_A\setminus \{x\}$ and $\pi^{\prime}_B \gets\pi_B\cup \{x\}$
		\STATE  Compute SRS value incrementally:\\
		$eva^{\prime}= SRS\left(eva,N^{(A)},N^{(B)},N^{(A)}_{mq},N^{(B)}_{mq}\right)$ 
		\RETURN $\pi^{\prime}, eva^{\prime}$
	\end{algorithmic}
\end{algorithm}

In the local search procedure shown in Algorithm \ref{alg: local_search}, we randomly assign an object $x$ to another cluster, and the partition $\pi$ will be updated to obtain a new $\pi^{\prime}$. Given a new partition $\pi^{\prime}$ and an old one $\pi$, we can compute the SRS value of $\pi^{\prime}$ quickly in an incremental way. Suppose a randomly selected object $x$ is moved from cluster $\pi_A$ to cluster $\pi_B$ in $\pi$. Then, the new $eva^{\prime}$ can be computed based on $eva$ by updating local terms in Equation \ref{eq: SRS}: $N^{(A)},N^{(B)},N^{(A)}_{mq},N^{(B)}_{mq}$, as shown in Lines \ref{algline: N}$\sim$\ref{algline: Nmq}. 

When $x$ is moved from $\pi_A$ to $\pi_B$ in Algorithm \ref{alg: local_search}, we can use $x_m$ to update $N^{(k)}_{mq}$($q=x_m$) for each attribute in these two clusters in $\mathcal{O}(1)$ when a hash table is utilized to store attribute values and their counts. Thus, we can find and update all local terms $N^{(A)}_{mq}$, $N^{(B)}_{mq}$ on $M$ attributes in $\mathcal{O}(M)$. Given these updated local terms and the old SRS value, we can obtain the new SRS value in $\mathcal{O}(1)$. Therefore, the time complexity of Algorithm \ref{alg: local_search} is $\mathcal{O}(M)$.

In the initialization step of Algorithm \ref{alg: sigcat} (Lines \ref{algline: int}$\sim$\ref{algline: endint}), we need to scan the data set once to obtain the frequency for each attribute value in order to calculate the SRS value, which requires $\mathcal{O}(NM)$ time.  In the iteration step (Lines \ref{algline: rpt}$\sim$\ref{algline: endrpt}), we need to execute the local search procedure in Algorithm \ref{alg: local_search} $\left(\mathcal{I}+N(K-1)\right)$ times, where $\mathcal{I}$ is the number of iterations before $N(K-1)$ times of consecutive failure on reducing the SRS value. Hence, the iteration step of Algorithm \ref{alg: sigcat} has a time complexity $\mathcal{O}(\mathcal{I}M+NMK)$. The overall time complexity of Algorithm \ref{alg: sigcat} is $\mathcal{O}(NM)+\mathcal{O}(\mathcal{I}M+NMK)=\mathcal{O}(\mathcal{I}M+NMK)$.

\subsection{The statistical significance of categorical clusters}
\label{sec: significance_clusters}
As a by-product, we can calculate a $p$-value from the SRS value to assess the statistical significance of the locally optimal partition $\pi^*$. To fulfill this task, the null distribution of SRS is needed to derive the $p$-value. Unfortunately, obtaining the exact null distribution of SRS analytically is difficult because $-2 \ln \Lambda^*$ does not follow a chi-square distribution in our context due to unidentifiable parameters for clusters \cite{Algeri2020}. Hence, we try to generate multiple reference partitions to construct the null distribution of SRS for deriving the empirical $p$-value.

To obtain a reference partition $\pi^\prime$ of $X$, we generate a randomized data set $X^\prime$ so that $\pi^\prime(X^\prime)$ can be obtained from $K$-SigCat. Here we introduce two methods for generating a randomized data set $X^\prime$ that in which the counts of all attribute values are preserved.

\begin{figure}[t]
	\includegraphics[scale=0.375]{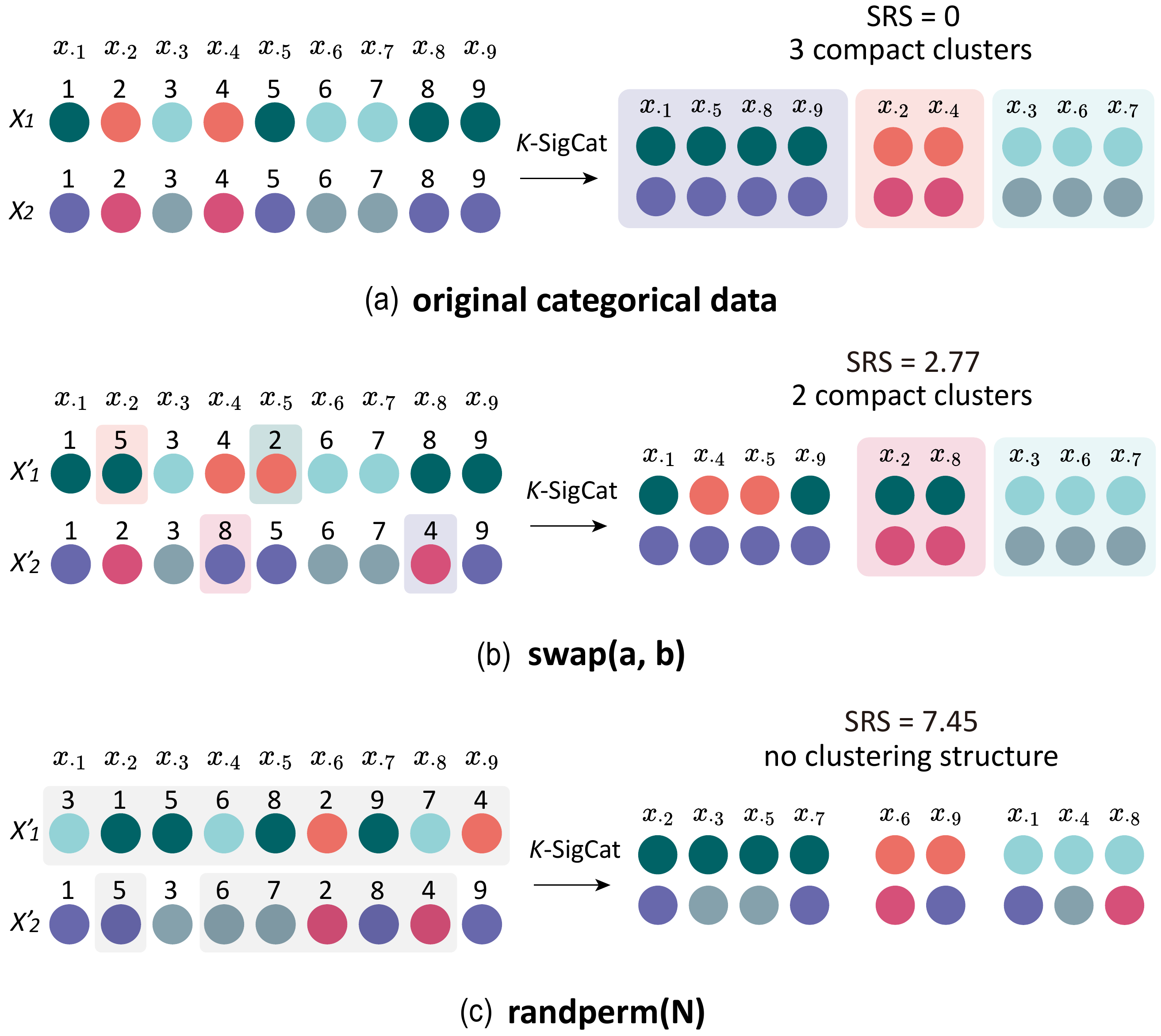}  
	\centering
	\caption{Two methods for generating randomized data sets from a categorical data. (a) An original categorical data $X=[X_1,X_2]$ with $N=9, Q_m=3 (m=1,2)$ and $K=3$. If we conduct the cluster analysis using $K$-SigCat, it is more likely to obtain a partition of three perfect clusters, as shown on the right side. (b) The \textbf{swap} method generates a randomized data by repeatedly exchanging two distinct attribute values in each attribute. This method is able to produce a data set that may partially maintain the cluster structure in the original data set (as shown on the right side). (c) The \textbf{randperm} method generates a randomized data set by re-ordering all attribute values for each attribute independently. This method is prone to produce a data set that is completely random with no clustering structure.} 
	\label{fig:random_data}   
\end{figure}

\begin{itemize}
	\item \textbf{swap(a,b)}: Given $X_m$, we randomly swap two distinct attribute values $x_{ma}$, $x_{mb}$. For example, in Fig.\ref{fig:random_data}(b), $x_{12}, x_{15}$ $(a=2,b=5)$ in $X_1$ and $x_{24},x_{28}$ $(a=4,b=8)$ in $X_2$ are exchanged. We have $x^\prime_{12}=x_{15},x^\prime_{15}=x_{12},x^\prime_{24}=x_{28},x^\prime_{28}=x_{24}$ and other values in $X^\prime=[X_1^\prime,X_2^\prime]$ and $X=[X_1,X_2]$ are the same.
	\item \textbf{randperm(N)}: The original $X_m$ is replaced with a new random permutation of all attribute values. For example, in Fig.\ref{fig:random_data}(c), we obtain a new data set by using two random permutations: $[3,1,5,6,8,2,9,7,4]$ in $X^\prime_1$ and $[1,5,3,6,7,2,8,4,9]$ in $X^\prime_2$.
\end{itemize}

Given a set of randomized data sets $R$, the empirical $p$-value for $\pi^*$ can be calculated as:
\begin{equation}
	\label{eq:pv}
	p\text{-value}= \frac{\lvert{\text{SRS}(X^\prime, \pi^\prime)\leq \text{SRS}(X, \pi^*) | X^\prime\in R}\rvert}{\lvert R \rvert},
\end{equation}
where each $\pi^\prime$ is the partition obtained by running $K$-SigCat on $X^\prime$. That is, the $p$-value is the proportion of randomized data sets in which a better clustering result than $\pi^*$ in terms of SRS value can be obtained.

\subsection{Determine the number of categorical clusters} 
We can plug the likelihood function into the Gap Statistic \cite{Tibshirani2001} to determine the number of clusters. More precisely, we define within-dispersion measures $W_k, k=2,\cdots,K_{max}$, as follows:
\begin{equation}
	W_k= \frac{1}{L(\boldsymbol{\theta_1}|X,\pi^*(k))},
\end{equation}
where a smaller $W_k$ value corresponds to a higher likelihood for a locally optimal $\pi^*$ with $k$ clusters. Then, the logarithmic $W_k$ can be written as:
\begin{equation}
	\ln W_k =\ln \frac{1}{L(\boldsymbol{\theta_1}|X,\pi^*(k))}=\text{SRS}\left(X,\pi^*(k)\right).
\end{equation}

Given a set of randomized data sets $R$, the Gap statistic can be calculated as:
\begin{equation}
	\label{eq:gap}
	\begin{split}
		\text{Gap}(k) & = (1/\lvert R \rvert)\sum_{X^\prime\in R}\ln W^\prime_k-\ln W_k \\
		& = (1/\lvert R \rvert)\sum_{X^\prime\in R}\text{SRS}\left(X^\prime,\pi^\prime(k)\right)-\text{SRS}\left(X,\pi^*(k)\right),
	\end{split}
\end{equation}
where each $\pi^\prime(k)$ is the partition obtained by running $K$-SigCat on $X^\prime$ with the input parameter $K=k$.

Here we determine the number of categorical clusters by finding the maximum of the following modified Gap statistic ($\text{Gap}^*$): 
\begin{equation}
	\label{eq: mGap}
	\hat{k}=\argmax_{k\in [2\colon K_{max}]}\frac{\text{Gap}(k)}{k\cdot SD},
\end{equation}
where $K_{max}$ is a user-specified parameter and SD is the standard derivation of $\text{SRS}\left(X^\prime, \pi^\prime(k)\right)$ for $X^\prime\in R$.

\section{Experiments}
\label{sec:experiments}
First, we evaluate the performance of $K$-SigCat\footnote{https://github.com/hulianyu/SigCat} and five categorical data clustering methods on nineteen real-world data sets in terms of two performance metrics. Next, we show the effectiveness of the empirical $p$-value on assessing the statistical significance of clustering results. Finally, we conduct some experiments to demonstrate that our method for cluster number estimation is comparable to existing methods. All experiments are conducted on an Intel i7-12700K@3.60GHz personal computer with 32G RAM.

We compare $K$-SigCat with both classic algorithms and state-of-the-art algorithms for clustering categorical data. 
\begin{itemize}
	\item $k$-modes \cite{Huang1998}: Similar to $k$-means, the $k$-modes algorithm is the de facto standard clustering algorithm for categorical data. 
	\item Entropy \cite{Li2004}: A entropy-based objective function and a corresponding clustering algorithm is presented in this paper. In the experiments, we try to minimize the entropy-based objective function using the same search procedure employed by our algorithm. 
	\item Distance Vectors (DV) \cite{Zhang2006}: It is the only significance-based method for categorical data clustering in the literature. The partition generated by the algorithm is unique, so we perform DV only once for each data set. The algorithm automatically determine the number of clusters. That is, the number of clusters cannot be manually specified in this algorithm. 
	\item Coupled Metric Similarity (CMS) \cite{Jian2018}: A similarity-based representation is proposed to capture the value-to-attribute-to-object heterogeneous coupling relationships. The parameter $\alpha$ in CMS is fixed to be $0.5$ as recommended by the authors. A similarity matrix produced by CMS is incorporated directly into spectral clustering \cite{Shi2000} in the experiments. 
	\item Novel Distance Weighting mechanism for Homogeneous Distance (HD-NDW) \cite{Zhang2022}: A similarity-based representation method is presented to learn weights of intra-attribute distances for nominal and ordinal attributes. Since $K$-SigCat does not consider the ordinal information of the attributes, we treat all attributes as non-ordered and the number of the ordinal attributes is fixed to be $0$ in HD-NDW for each data set. In this setting, \cite{Zhang2022} reports that HD-NDW still outperforms other state-of-the-art clustering methods such as WOC \cite{Jia2017}, SBC \cite{Qian2015} and CDE \cite{Jian2018a}.
	\item Graph representation (CDC\_DR) \cite{Bai2022}: A similarity-based representation framework is developed to learn and integrate vectors of categorical values from their graph structure. Autoencoder \cite{Hinton2006} and joint operation are used in the experiments.
\end{itemize}

The source codes of above methods (except for the entropy-based method) are publicly available and the corresponding URLs are provided in the \ref{sec:appendixB}. The implementation of the entropy-based method is provided in our source code. In performance comparison, for each algorithm on each data set, we run 50 executions independently and the average results are reported. The number of clusters $K$ used in all algorithms (except for DV) is set to be the ground-truth cluster number on each data set.
\subsection{Data sets and performance metrics}
All real-world data sets used in the experiments can be downloaded from the UCI Machine Learning Repository \cite{Dua2017}. By treating each missing value in each attribute as a unique value, the properties of 19 data sets are shown in Table \ref{tab: datasets}. Consistent with the notation used in the previous section, $Q$ denotes the number of all distinct values in $M$ attributes of a data set. 

\begin{table}[t]
	\small
	\caption{The properties of 19 data sets from UCI repository.}
	\label{tab: datasets}
	\centering
	\begin{tabular}{lccccc}
		\toprule
		Data set & \textit{abbr.} & $N$ & $M$ & $Q$ & $K$ \\
		\midrule
		Lenses & Le & 24 & 4 & 9 & 3 \\
		Soybean (Small) & So & 47 & 21 & 58 & 4 \\
		Lung Cancer & Lc & 32 & 56 & 159 & 3 \\
		Zoo & Zo & 101 & 16 & 36 & 7 \\
		Promoter Sequences & Ps & 106 & 57 & 228 & 2 \\
		Hayes-Roth & Hr & 132 & 4 & 15 & 3 \\
		Lymphography & Ly & 148 & 18 & 59 & 4 \\
		Heart Disease & Hd & 303 & 13 & 57 & 5 \\
		Solar Flare & Sf & 323 & 9 & 25 & 6 \\
		Primary Tumor & Pt & 339 & 17 & 42 & 21 \\
		Dermatology & De & 366 & 33 & 129 & 6 \\
		House Votes & Hv & 435 & 16 & 48 & 2 \\
		Balance Scale & Bs & 625 & 4 & 20 & 3 \\
		Breast Cancer & Bc & 699 & 9 & 90 & 2 \\
		Tic-Tac-Toe & Tt & 958 & 9 & 27 & 2 \\
		Car Evaluation & Ce & 1728 & 6 & 21 & 4 \\
		Chess (kr vs kp) & Ch & 3196 & 36 & 73 & 2 \\
		Mushroom & Mu & 8124 & 20 & 111 & 2 \\
		Nursery & Nu & 12960 & 8 & 27 & 5 \\
		\bottomrule
	\end{tabular}
\end{table}

To compare the performance of different categorical data clustering methods, we use two widely used external metrics: Clustering Accuracy (ACC), Normalized Mutual Information (NMI). The external metrics estimate the difference between the cluster label assigned by each clustering algorithm and the label provided by ground truth. The larger these metrics are, the better the performance of a clustering algorithm. 

The ACC is defined as \cite{Cai2005}:
\begin{equation}
	\text{ACC}=\frac{1}{N}\sum_{i=1}^{N} \delta \left(t_i, \text{map}(c_i)\right), 
\end{equation}
where $t_{i}$ is the ground-truth label for the $i$-th categorical object, $c_i$ is the corresponding cluster label produced by a clustering algorithm, $\text{map}(c_i)$ is a mapping function that maps $c_i$ to an equivalent ground-truth label, $\delta (x,y)=1$ if $x=y$ and $0$ otherwise. The best mapping can be obtained by using Kuhn-Munkres algorithm \cite{Lovasz2009}.

Suppose $\pi$ is a set of clusters produced by a clustering algorithm and $\pi_{GT}$ is the set of ground-truth clusters. The NMI is defined as:
\begin{equation}
	\text{NMI}=\frac{H(\pi)+H(\pi_{GT})-H(\pi,\pi_{GT})}{\left(H(\pi)+H(\pi_{GT})\right)/2},
\end{equation}
where $H(\pi)$ and $H(\pi_{GT})$ denote the entropies of $\pi$ and $\pi_{GT}$ respectively, and $H(\pi,\pi_{GT})$ denotes the corresponding joint entropy.

\subsection{Performance comparison}
The performance comparison results in term of two metrics are listed in Table \ref{tab:comparsion}, where the mean and average rank describe the overall performance of each method. The best and the second best results are marked in boldface. The running time comparison of different algorithms is displayed in Fig.\ref{fig:time-heatmap}. From Table \ref{tab:comparsion} and Fig.\ref{fig:time-heatmap}, we have the following important observations.

\newcommand{\NA}{---}
\begin{table}[p]
	\small
	\caption{The performance comparison of $K$-SigCat, Entropy, DV, $k$-modes, CMS and HD-NDW in terms of ACC and NMI on 19 data sets. We execute each algorithm (except for DV) 50 times and report its average results. DV is unable to find statistically significant clusters on some data sets, which we denote as ``---". For CMS and HD-NDW, the parameters are fixed to be $\alpha=0.5$ and $0$, respectively.}
	\label{tab:comparsion}
	\begin{adjustbox}{width=0.9\textwidth, center} 
		\begin{tabular}{l|cccccccc}
			\toprule
			Metric & Data set&$K$-SigCat&Entropy&DV&$k$-modes&CMS&HD-NDW&CDC\_DR\\
			\midrule
			\multirow{19}{*}{ACC}   
			& Le & 0.537 & 0.545 & \NA & 0.458 & \textbf{0.581} & \textbf{0.548} & 0.488 \\
			& So & \textbf{0.936} & 0.869 & \textbf{1} & 0.702 & \textbf{1} & 0.816 & 0.896 \\
			& Lc & \textbf{0.603} & 0.582 & \NA & 0.563 & \textbf{0.583} & 0.538 & 0.537 \\
			& Zo & \textbf{0.753} & 0.715 & \textbf{0.950} & 0.505 & 0.642 & 0.727 & 0.681 \\
			& Ps & 0.720 & \textbf{0.744} & \NA & 0.557 & \textbf{0.766} & 0.681 & 0.578 \\
			& Hr & \textbf{0.441} & 0.422 & \NA & 0.364 & 0.341 & 0.412 & \textbf{0.491} \\
			& Ly & \textbf{0.530} & 0.524 & 0.345 & 0.486 & 0.505 & \textbf{0.680} & 0.474 \\
			& Hd & 0.381 & 0.362 & 0.168 & 0.317 & 0.317 & \textbf{0.442} & \textbf{0.390} \\
			& Sf & 0.500 & 0.447 & \NA & \textbf{0.560} & 0.439 & \textbf{0.514} & 0.402 \\
			& Pt & 0.298 & \textbf{0.301} & \NA & 0.283 & 0.263 & \textbf{0.305} & 0.281 \\
			& De & 0.701 & 0.617 & 0.497 & 0.577 & \textbf{0.808} & \textbf{0.707} & 0.683 \\
			& Hv & \textbf{0.888} & 0.871 & 0.756 & 0.867 & \textbf{0.878} & 0.867 & 0.836 \\
			& Bs & \textbf{0.446} & 0.440 & \NA & 0.419 & 0.382 & 0.432 & \textbf{0.455} \\
			& Bc & \textbf{0.993} & 0.964 & 0.738 & 0.913 & 0.964 & \textbf{0.969} & 0.787 \\
			& Tt & 0.566 & 0.562 & \NA & 0.603 & 0.517 & 0.572 & 0.557 \\
			& Ce & 0.362 & 0.383 & \NA & \textbf{0.414} & 0.399 & 0.367 & \textbf{0.400} \\
			& Ch & 0.628 & 0.603 & \NA & \textbf{0.697} & 0.648 & \textbf{0.679} & 0.628 \\
			& Mu & 0.751 & 0.720 & 0.245 & 0.712 & \textbf{0.892} & \textbf{0.769} & 0.734 \\
			& Nu & 0.309 & 0.307 & \NA & \textbf{0.499} & 0.289 & \textbf{0.321} & 0.318 \\
			\midrule
			\multicolumn{2}{l}{Mean} & \textbf{0.597} & 0.578 & \NA & 0.552 & \textbf{0.590} & \textbf{0.597} & 0.559 \\
			\multicolumn{2}{l}{Average Rank} & \textbf{2.7} & 3.7  & \NA & 4.2 & 3.5 & \textbf{2.8} & 4.1 \\
			\midrule
			\multirow{19}{*}{NMI}  
			& Le & 0.235 & 0.278 & \NA & \textbf{0.366} & \textbf{0.281} & 0.249 & 0.178 \\
			& So & \textbf{0.919} & 0.872 & \textbf{1} & 0.721 & \textbf{1} & 0.844 & 0.880 \\
			& Lc & \textbf{0.322} & \textbf{0.293} & \NA & 0.221 & 0.237 & 0.217 & 0.210 \\
			& Zo & \textbf{0.785} & 0.756 & \textbf{0.916} & 0.614 & 0.653 & 0.764 & 0.694 \\
			& Ps & 0.194 & \textbf{0.265} & \NA & 0.010 & \textbf{0.298} & 0.149 & 0.032 \\
			& Hr & \textbf{0.050} & \textbf{0.050} & \NA & 0.009 & 0.000 & \textbf{0.031} & 0.362 \\
			& Ly & \textbf{0.218} & 0.205 & 0.161 & 0.161 & 0.183 & \textbf{0.259} & 0.150 \\
			& Hd & \textbf{0.179} & \textbf{0.178} & 0.153 & 0.140 & 0.130 & 0.175 & 0.166 \\
			& Sf & 0.348 & 0.277 & \NA & \textbf{0.355} & 0.233 & \textbf{0.398} & 0.211 \\
			& Pt & \textbf{0.366} & \textbf{0.364} & \NA & 0.312 & 0.311 & \textbf{0.366} & 0.332 \\
			& De & \textbf{0.828} & 0.760 & 0.275 & 0.605 & \textbf{0.805} & 0.776 & 0.786 \\
			& Hv & \textbf{0.479} & 0.474 & 0.331 & 0.477 & \textbf{0.490} & 0.461 & 0.403 \\
			& Bs & 0.027 & 0.032 & \NA & 0.008 & 0.015 & \textbf{0.035} & \textbf{0.041} \\
			& Bc & \textbf{0.836} & 0.779 & 0.371 & 0.554 & 0.778 & \textbf{0.785} & 0.253 \\
			& Tt & 0.007 & \textbf{0.010} & \NA & \textbf{0.018} & 0.000 & 0.007 & 0.006 \\
			& Ce & 0.038 & 0.071 & \NA & 0.051 & \textbf{0.120} & 0.080 & \textbf{0.086} \\
			& Ch & \textbf{0.092} & 0.088 & \NA & 0.000 & 0.087 & 0.057 & \textbf{0.120} \\
			& Mu & 0.260 & 0.243 & 0.226 & 0.226 & \textbf{0.552} & \textbf{0.362} & 0.264 \\
			& Nu & 0.068 & 0.073 & \NA & \textbf{0.331} & 0.033 & \textbf{0.104} & 0.083 \\
			\midrule
			\multicolumn{2}{l}{Mean} & \textbf{0.329} & 0.319 & \NA & 0.273 & \textbf{0.327} & 0.322 & 0.277 \\
			\multicolumn{2}{l}{Average Rank} & \textbf{2.6} & \textbf{3.2}  & \NA & 4.4 & 3.6 & \textbf{3.2} & 3.9 \\
			\bottomrule
		\end{tabular} 
	\end{adjustbox}
\end{table}

\begin{figure}[h]
	\includegraphics[scale=0.45]{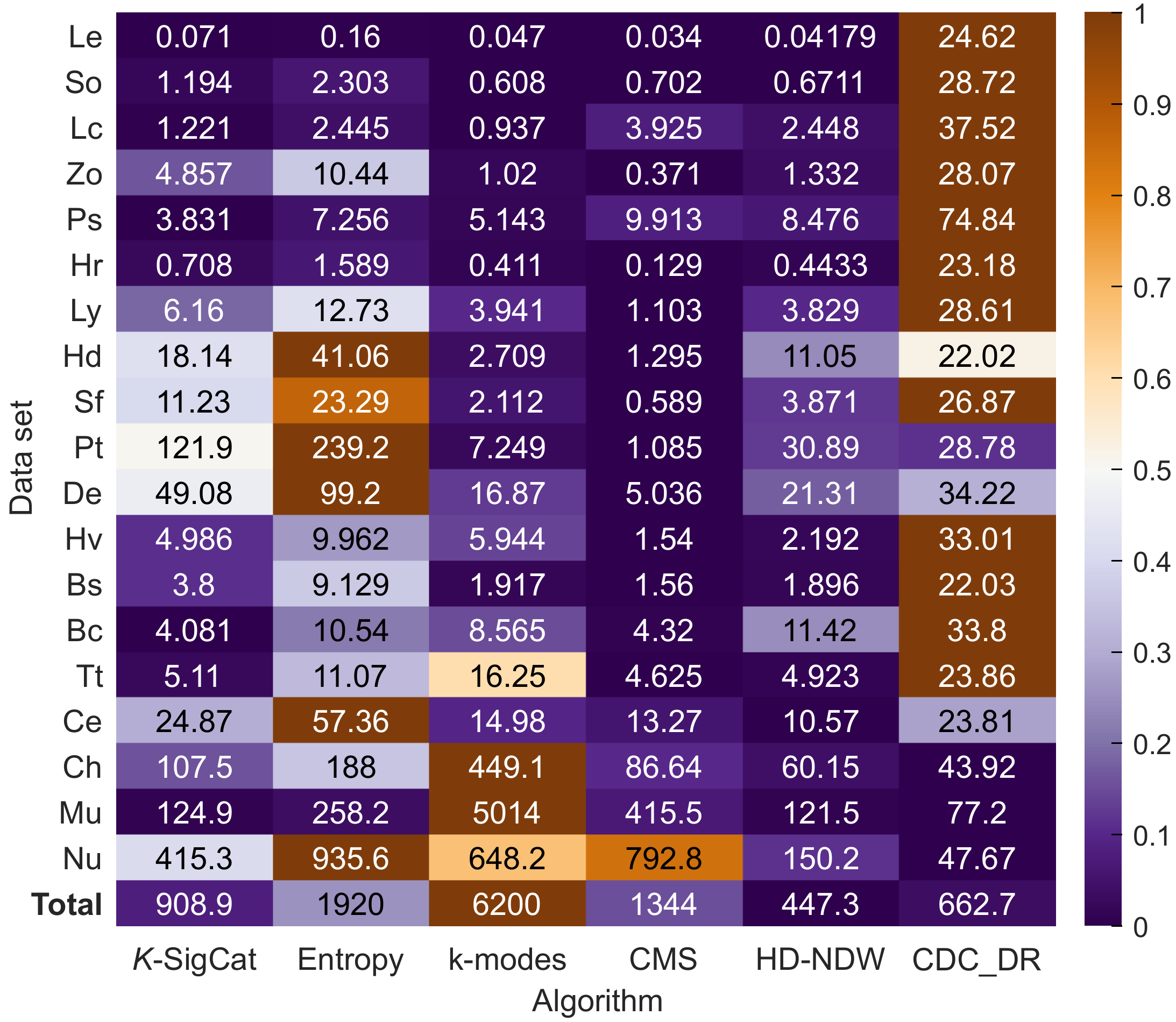}  
	\centering
	\caption{Running time for 50 executions of each algorithm on nineteen data sets. The DV algorithm is not included in this map since it is too time consuming, taking about 8.6 hours to execute once on all data sets. The running time on each data set is normalized into $(0,1]$, as shown in the color map on the right side.} 
	\label{fig:time-heatmap}   
\end{figure}

\begin{itemize}
	\item \ul{Overall performance}: $K$-SigCat can achieve the best overall performance in term of two metrics and acceptable overall performance with respect to the running time, which means that the significance-based objective function and our search procedure are effective. In particular, our algorithm achieves the best or the second best results on 7 data sets (So, Lc, Zo, Hr, Ly, Hv, Bc) in term of both two metrics. Compared with the classic $k$-modes algorithm, our algorithm achieves an overall improvement of more than $8\%$ in ACC and $20\%$ in NMI. 
	\item \ul{Comparison with state-of-the-art methods}: In term of two metrics, $K$-SigCat can achieve a better performance than any method from CMS, HD-NDW and CDC\_DR on 8 data sets (Lc, Zo, Hd, Pt, De, Hv, Bc, Tt), while CMS or HD-NDW outperforms $K$-SigCat on 8 data sets (Le, So, Ps, Ly, Sf, Ce, Mu, Nu). One possible reason is that the similarity metrics learned by these two methods may cause a loss of information. The performance of CMS and HD-NDW could be further improved by using fine-tuned parameters. However, it can be quite time-consuming and the domain knowledge on each data may be required.
	\item \ul{Comparison with entropy-based method}: In term of all metrics, $K$-SigCat can beat the entropy-based method on most of data sets. Since the same local search procedure is also used in the entropy-based method, the difference can be mainly attributed to the employed objective functions. The entropy-based objective function is more complex than SRS. This could partially explain why entropy-based method takes more than twice the running time of $K$-SigCat on most data sets.
	\item \ul{Comparison with DV}: $K$-SigCat spends much less running time than that of DV on all data sets. In term of performance metrics, DV only outperforms $K$-SigCat on So and Zo. DV fails to report the clustering results on some data sets because it cannot identify any statistically significant clusters.
\end{itemize}

\subsection{The statistical significance of a partition}
\label{sec: test_sig}

To calculate an empirical $p$-value according to Equation \ref{eq:pv}, we need to generate a group of randomized data sets $R$. In the experiment, each $X^\prime\in R$ is generated by using the swap method presented in Section \ref{sec: significance_clusters}. More specifically, for each attribute, we use the one-time swap operation to generate $X^\prime_m$ from a given $X_m$. A partition $\pi^\prime$ and its $\text{SRS}^\prime$ value on $X^\prime$ are obtained by using $K$-SigCat. In the following experiments, the group size is fixed to be $\lvert R\rvert=100$, i.e., $100$ randomized data sets are generated. For each data set, we independently generate $50$ groups of randomized data sets and calculate $50$ empirical $p$-values.

\subsubsection{Synthetic data sets}
We first conduct experiments on simulation data sets with no clustering structure. We generate such a synthetic data set according to several user-specified parameters: the number of objects $N$, the number of attributes $M$, the number of distinct feature values and the count of each feature value for each attribute. More precisely, we utilize a given data set $X$ as a template and generate a completely randomized data set $X^{s}$ by using the randperm operation to generate $X_m^{s}$ from a given $X_m$ on each attribute. For a synthetic data set $X^{s}$, a locally optimal partition $\pi^s$ can be obtained by running $K$-SigCat when $K$ is specified as the ground-truth cluster number of $X$. We can obtain the empirical $p$-value of $\pi^s$ on $X^{s}$ by generating a set of randomized data sets $R^s$. For each one of 19 UCI data sets, we construct a synthetic data set accordingly which is denoted by X(s). 

\begin{figure}[t]
	\includegraphics[scale=0.42]{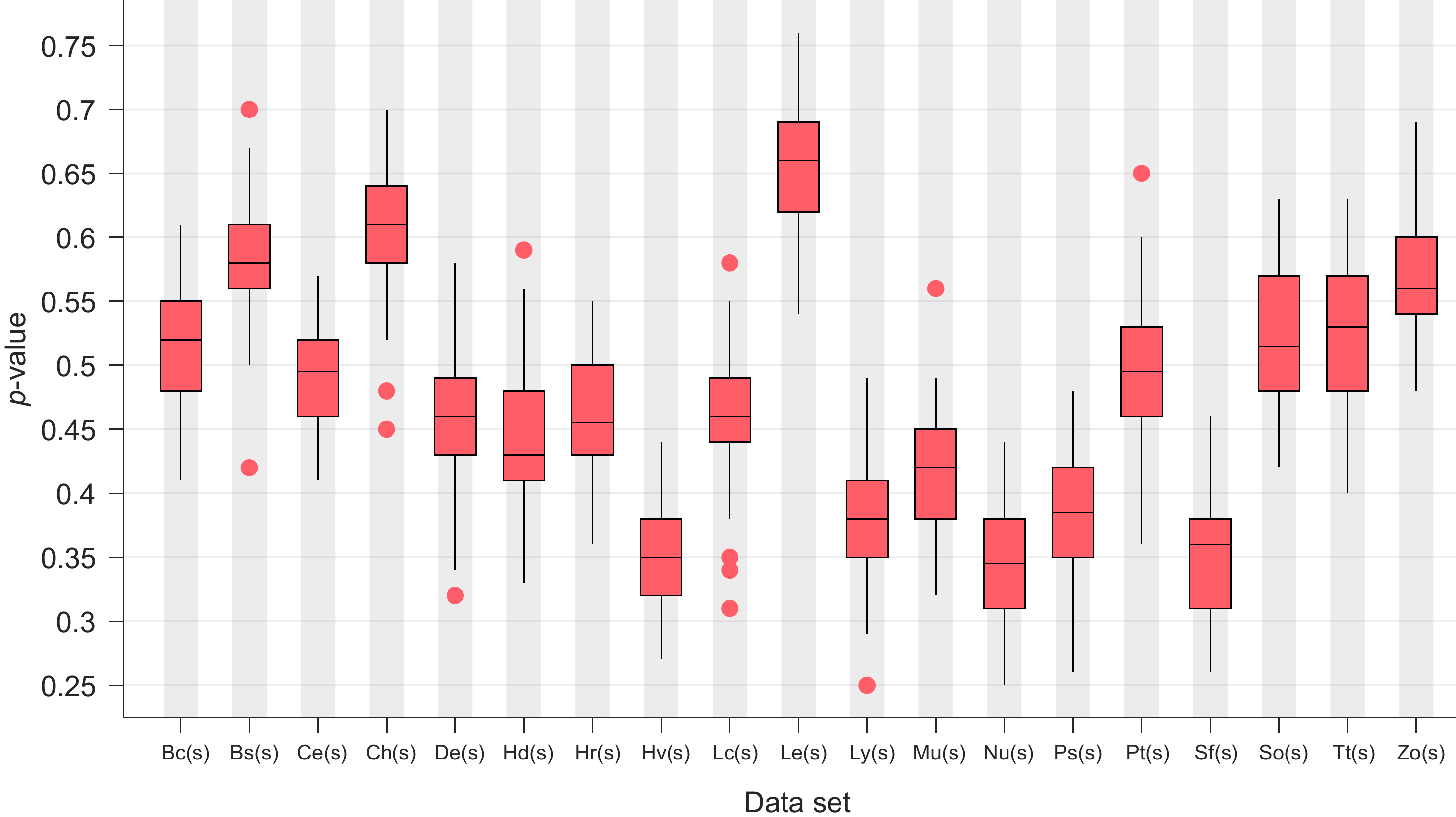}  
	\centering
	\caption{The box-plots of empirical $p$-values for $\pi^s$ on 19 synthetic data sets. For each data set $X^s$, we independently generate $50$ groups of randomized data sets to calculate the $p$-values.} 
	\label{fig:pxr}   
\end{figure}

\begin{figure}[t]
	\includegraphics[scale=0.42]{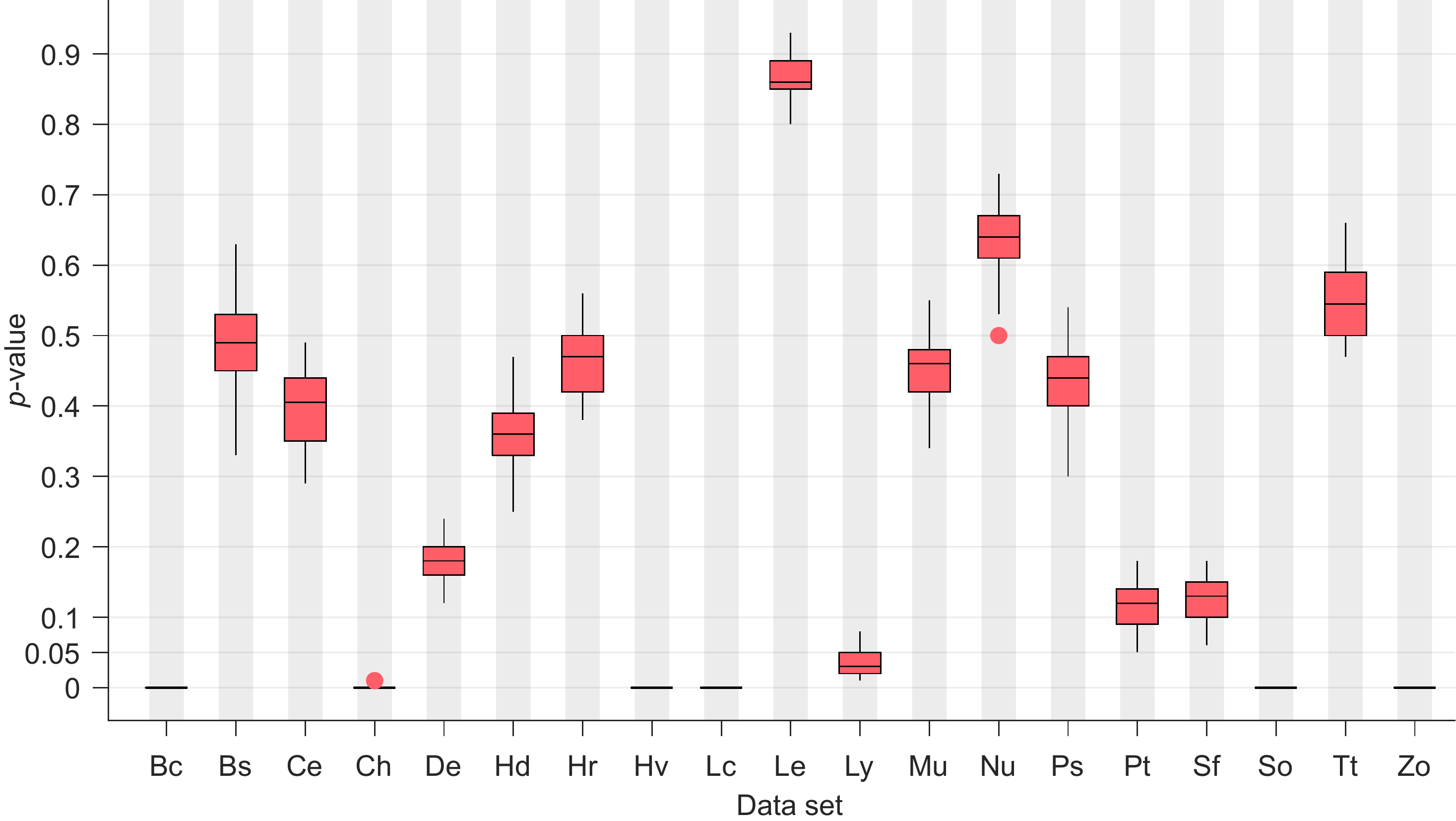}  
	\centering
	\caption{The box-plots of empirical $p$-values for $\pi^*$ on 19 real-world data sets. For each data set $X$, we independently generate $50$ groups of randomized data sets to calculate $p$-values.} 
	\label{fig:px}   
\end{figure} 

The significance testing results on the partitions from these simulated data sets are presented in Fig.\ref{fig:pxr}. As expected, all empirical $p$-values are quite large, indicating that each partition obtained from such synthetic data is not statistically meaningful.  

\subsubsection{Real data sets}
Next, we calculate the $p$-value on the partition obtained from real UCI data sets with known clustering structure. As shown in Fig.\ref{fig:px}, not all partitions produced by $K$-SigCat on these data sets are statistically significant if the significance level is specified to be $0.05$. More precisely, only the clustering results on Bc, Ch, Hv, Lc, Ly, So and Zo are statistically significant and we have the following key findings.

\begin{itemize}
	\item \ul{The consistency with DV:} At least one significant cluster is reported by DV on So, Zo, Hv, Bc, De, Ly, Mu and Hd. Although the statistical significance of an individual cluster is not equivalent to that of an entire partition found by $K$-SigCat, we can still partially interpret why the results of DV on some data sets are better. DV performs relatively well on So, Zo, Hv and BC with $\text{ACC}>0.7$, because these data sets have a statistically significant clustering structure. 
	\item \ul{Comparison between significant and non-significant partitions:} According to the ACC on each data set, the first four best partitions (on Bc, So, Hv and Zo) are all statistically significant while the last seven worst partitions (on Pt, Nu, Ce, Hd, Hr, Bs, and Sf) are all not statistically significant. The statistically significant partition is more likely to correlate with a good clustering structure. For example, average values of ACC and NMI for those partitions with $p\text{-values}>0.05$ are $0.501$ and $0.217$, respectively. In contrast, average ACC and NMI for statistically significant partitions are $0.762$ and $0.522$, respectively. A partition that is not statistically significant may indicate the non-existence of a clustering structure for the target data set, as the empirical $p$-value is almost indistinguishable from that on a completely randomized data set.
\end{itemize}

\subsection{Estimate the number of categorical clusters}
To determine the number of categorical clusters $\hat{k}$, we compare existing methods with our $\text{Gap}^*$ method in Equation \ref{eq: mGap}. In the following experiments, we generate a set of randomized data sets $R$ in the same way as Section \ref{sec: test_sig}. The size of $R$ is fixed to be $\lvert R\rvert=20$ and the parameter $K_{max}$ is fixed to be $10$ (with one exception: the $K_{max}$ is set to be $30$ on the Primary Tumor data set). For each data set, we independently generate $50$ groups of randomized data sets and estimate $\hat{k}$ with $50$ times.

\begin{table}[p]
	\caption{The estimated number of categorical clusters on 19 UCI data sets. $K$ is the ground-truth cluster number. The results of the $\text{Gap}^*$ method and two baselines are calculated based on the same batch of $\pi^*(k)$ with varied $k$ on each data set. For each data set, the average value of $50$ estimated cluster numbers by using the $\text{Gap}^*$ method is shown in parentheses. The results of average difference between $\hat{k}$ and the ground-truth cluster number are shown in the last line.}
	\label{tab: clusternumber}
	\centering
	\begin{tabular}{cclcc}
		\toprule
		Data & $K$ & $\text{Gap}^*$ &  BIC & BKPlot \\
		\midrule
		Le & 3 & \textbf{3} (2.73) & 2 & 2 \\
		So & 4 & \textbf{4} (5.25) & 2 & 3 \\
		Lc & 3 & \textbf{3} (2.66) & 2 & 2 \\
		Zo & 7 & \textbf{7} (6.96) & 3 & 3 \\
		Ps & 2 & 3 (2.94) & \textbf{2} & \textbf{2} \\
		Hr & 3 & 2 (3.10) & 4 & \textbf{3} \\
		Ly & 4 & \textbf{4} (3.34) & 2 & 3 \\
		Hd & 5 & 2 (2.50) & 2 & 3 \\
		Sf & 6 & 2 (2.40) & 5 & 3 \\
		Pt &21 & 2 (2.04) & 2 & 4 \\
		De & 6 & 2 (2.06) & 3 & 2 \\
		Hv & 2 & \textbf{2} (\textbf{2}) & 3 & \textbf{2} \\
		Bs & 3 & \textbf{3} (5.48) & 9 & 5 \\
		Bc & 2 & \textbf{2} (\textbf{2}) & \textbf{2} & \textbf{2} \\
		Tt & 2 & \textbf{2} (2.54) & 9 & 3 \\
		Ce & 4 & 5 (3.94) & 10 & 2 \\
		Ch & 2 & \textbf{2} (\textbf{2}) & 10 & 3 \\
		Mu & 2 & 6 (4.76) & 10 & 5 \\
		Nu & 5 & 3 (4.28) & 10 & 7 \\
		\midrule
		\multicolumn{2}{l}{Average Difference} & 2.05 (2.06) & 4.11 & 2.37 \\
		\bottomrule
	\end{tabular}
\end{table}

Since the standard Gap Statistic method \cite{Tibshirani2001} is developed to handle numerical data sets, it is not included in the performance comparison. The following two baselines are employed in the experiment: Bayesian Information Criterion (BIC) \cite{Hastie2009}, Best-K Plot (BKPlot) \cite{Chen2009}.

In our model, we have a user-specified $k$, the sample size $\mathbb{N}=M\cdot N$, the total number of parameters $\mathbb{D}=k\cdot\sum\limits_{m=1}^{M}Q_m$, and the maximum log-likelihood value $\mathbb{L^*}=-\text{SRS}\left(X,\pi^*(k)\right)$, where $\text{SRS}\left(X,\pi^*(k)\right)$ is obtained by using $K$-SigCat when $K=k$. Then, given $X$ and $k$, the BIC can be calculated as:
\begin{equation}
	\label{eq: BIC}
	\begin{aligned}
		\text{BIC}(k) & =-2\cdot \mathbb{L^*}(k)+\mathbb{D}\cdot \ln \mathbb{N}\\
		& = 2\cdot \text{SRS}\left(X,\pi^*(k)\right) +(k\cdot\sum\limits_{m=1}^{M}Q_m)\cdot \ln MN.
	\end{aligned}
\end{equation}
The number of categorical clusters $\hat{k}$ is one of $k\in [2\colon K_{max}]$ that can achieve  the minimum $\text{BIC}(k)$.

As critical knees in BKPlot, the second-order difference is defined as:
\begin{equation}
	\label{eq: bk}
	B(k) =\beta^2\mathbb{S}(k) = \beta\mathbb{S}(k-1) - \beta\mathbb{S}(k),
\end{equation}
where $\beta\mathbb{S}(k)=\mathbb{S}(k)-\mathbb{S}(k+1)$ and $\mathbb{S}(k)=\text{SRS}\left(X,\pi^*(k)\right)-\text{SRS}\left(X,\pi^*(k+1)\right)$. To determine the best $\hat{k}$ from multiple critical knees in BKPlot, we use one of $k\in [2\colon K_{max}]$ that has the maximum $B(k)$. 

For the $\text{Gap}^*$ method, we report the most frequent $\hat{k}$ and the average value from $50$ estimated cluster numbers. Given $\pi^*(k)$ with different $k\in [2\colon K_{max}]$ on each data set, the results of the $\text{Gap}^*$ method and three baselines are shown in Table \ref{tab: clusternumber}. The estimated cluster numbers are marked in boldface when $\hat{k}=K$. We observe that $\text{Gap}^*$ can find the ground-truth cluster number based on the most frequent $\hat{k}$ on 10 data sets: Le, So, Lc, Zo, Ly, Hv, Bs, Bc, Tt and Ch. $\text{Gap}^*$ also achieves the best overall performance in terms of average difference.

\section{Conclusions and Future Work}
\label{sec:conclusion}
To conduct cluster analysis on categorical data set in a statistically sound manner, we propose to assess the statistical significance of a set of clusters using the likelihood ratio test. As a consequence, we obtain a significance-based objective function for clustering validation and a corresponding clustering algorithm $K$-SigCat. In addition, we further present methods on how to solve the clusterability problem by calculating an empirical $p$-value and tackle the cluster number estimation issue by plugging the test statistic into the well-known gap statistic method. A series of experiments on both real and simulated data sets verify the effectiveness and efficiency of our method on cluster detection, clustering tendency assessment and cluster number estimation.

In the field of categorical data clustering, there is still a lack of effective methods on assessing the statistical significance of categorical clusters. Therefore, our method has the following salient features: (1) It can help to justify whether a given categorical data set has a statistically significant clustering structure or not. (2) It can help to determine whether a partition reported by a third-party clustering algorithm is statistically significant or not. 

There are two issues in our method that need to be further addressed. Firstly, calculating an empirical $p$-value requires generating multiple reference partitions, which is quite time-consuming especially for large categorical data sets. Secondly, there is no theoretical guarantee to obtain the globally optimal partition, since a Monte Carlo search procedure is used in $K$-SigCat.

In future work, we will investigate how to analytically calculate the $p$-value of a partition such that it can be directly employed as the objective function. However, this is a challenging issue since it is difficult to obtain the exact null distribution of LR test statistics in the context of cluster analysis. Hence, we will further investigate the potential of utilizing other types of significance tests to derive an analytical $p$-value for one cluster or a set of clusters. 

\section*{Acknowledgment}
This work has been partially supported by the Natural Science Foundation of China under Grant No. 61972066.





\appendix
\section{The upper and lower bounds of EE in terms of SRS}
\label{sec:appendixA}
EE can be written as:
\begin{equation*}
	\begin{split}
		\text{EE}(X,\pi) & = -\frac{1}{N}\sum\limits_{m=1}^{M}\sum\limits_{k=1}^{K}\sum\limits_{q=1}^{Q_m}\big({N_{mq}^{(k)}}\cdot \ln\frac{{N_{mq}^{(k)}}}{N^{(k)}} + (N^{(k)}-N_{mq}^{(k)})\cdot \ln\frac{N^{(k)}-N_{mq}^{(k)}}{N^{(k)}}\big)\\
		&= -\frac{1}{N}\sum\limits_{m=1}^{M}\sum\limits_{k=1}^{K}\sum\limits_{q=1}^{Q_m}(\cdots)
	\end{split}
\end{equation*}
\begin{equation}
	\label{eq: EE_SRS}
	\begin{split}
		\cdots & = N_{mq}^{(k)}\cdot \ln N_{mq}^{(k)}-N_{mq}^{(k)}\cdot \ln N^{(k)} + (N^{(k)}-N_{mq}^{(k)})\cdot \left(\ln(N^{(k)}-N_{mq}^{(k)})-\ln(N^{(k)})\right)\\
		&=N_{mq}^{(k)}\cdot\ln N_{mq}^{(k)}-N_{mq}^{(k)}\cdot \ln N^{(k)} + N^{(k)}\cdot \ln(N^{(k)}-N_{mq}^{(k)})-N^{(k)}\cdot \ln N^{(k)}\\
		&-N_{mq}^{(k)}\cdot \ln(N^{(k)}-N_{mq}^{(k)})+ N_{mq}^{(k)}\cdot \ln N^{(k)}\\
		&= N_{mq}^{(k)}\cdot \ln N_{mq}^{(k)}-N^{(k)}\cdot \ln N^{(k)} + (N^{(k)}-N_{mq}^{(k)})\cdot \ln(N^{(k)}-N_{mq}^{(k)})
	\end{split}
\end{equation}

We rewrite $\text{SRS}$ and its variant $\text{SRS}^{\prime}$ as:
\begin{equation}
	\label{eq: rw1_SRS}
	\text{SRS}(X,\pi, N^{(k)}_{mq}) = M\cdot\sum\limits_{k=1}^{K}\big(N^{(k)}\cdot \ln {N^{(k)}}\big) - \sum\limits_{m=1}^{M}\sum\limits_{k=1}^{K}\sum\limits_{q=1}^{Q_m}\big({N_{mq}^{(k)}}\cdot \ln {N_{mq}^{(k)}}\big),
\end{equation}
\begin{equation}
	\label{eq: rw2_SRS}
	\begin{split}
	\text{SRS}^{\prime}(X,\pi, N^{(k)}-N^{(k)}_{mq}) & = M\cdot\sum\limits_{k=1}^{K}\big(N^{(k)}\cdot \ln {N^{(k)}}\big)\\ 
	& - \sum\limits_{m=1}^{M}\sum\limits_{k=1}^{K}\sum\limits_{q=1}^{Q_m}\big(({N^{(k)}-N_{mq}^{(k)}}) \cdot \ln ({N^{(k)}-N_{mq}^{(k)}})\big).
	\end{split}
\end{equation}

Then, bring SRS and its variant into EE, we have:
\begin{equation}
	\text{EE} = \frac{1}{N}(\text{SRS}+\text{SRS}^{\prime}) + \frac{1}{N}\sum\limits_{k=1}^{K}\big(\sum\limits_{m=1}^{M}\sum\limits_{q=1}^{Q_m}N^{(k)}\cdot \ln N^{(k)} - (2M\cdot N^{(k)}\cdot \ln{N^{(k)}})\big),
\end{equation}
where $N, M$ and $K, N^{(k)}$ are easy to be obtained from $X$ and $\pi$ respectively. The relation between EE and SRS is as follows:
\begin{equation}
	\label{eq: re_SRS}
	\text{EE} -\frac{1}{N}(\text{SRS}+\text{SRS}^{\prime})=\frac{1}{N}\sum\limits_{k=1}^{K}\big(\sum\limits_{m=1}^{M}\sum\limits_{q=1}^{Q_m}N^{(k)}\cdot \ln N^{(k)} - (2M\cdot N^{(k)}\cdot \ln{N^{(k)}})\big).
\end{equation}

Hence, we can obtain the connection between EE and SRS under several special cases:
\begin{itemize}
	\item If $\sum\limits_{m=1}^{M}\sum\limits_{q=1}^{Q_m}1=2M$, we have $\text{EE} = \frac{1}{N}(\text{SRS}+\text{SRS}^{\prime})$.
	\item If $\sum\limits_{m=1}^{M}\sum\limits_{q=1}^{Q_m}1>2M$, we have $\text{EE} > \frac{1}{N}(\text{SRS}+\text{SRS}^{\prime})$.
	\item If $\sum\limits_{m=1}^{M}\sum\limits_{q=1}^{Q_m}1<2M$, we have $\text{EE} < \frac{1}{N}(\text{SRS}+\text{SRS}^{\prime})$.
\end{itemize}

Finally, given any $X$ and $\pi$, the number of distinct attribute values for each attribute in each cluster should fall into the interval $[1,N/K]$. The upper and lower bounds of EE in terms of SRS is:
\begin{equation*}
	\frac{1}{N}(\text{SRS}+\text{SRS}^{\prime}) +\frac{1}{N}\sum\limits_{k=1}^{K}(\frac{N}{K}-2M)\cdot N^{(k)}\cdot \ln{N^{(k)}}
\end{equation*}
\begin{equation}
	\geq EE \geq\\
\end{equation}
\begin{equation*}
	\frac{1}{N}(\text{SRS}+\text{SRS}^{\prime}) +\frac{1}{N}\sum\limits_{k=1}^{K}(-M)\cdot N^{(k)}\cdot \ln{N^{(k)}}.
\end{equation*}

\section{The URLs for compared algorithms in the experiments}
\label{sec:appendixB}
All source codes are implemented in Matlab. $k$-modes and DV are obtained from \url{https://github.com/hetong007/CategoricalClustering}, CMS is obtained from \url{https://github.com/jiansonglei/CMS}, HD-NDW is obtained from \url{https://www.comp.hkbu.edu.hk/~ymc/papers/journal/Source_code/TPAMI-2021-3056510.zip}, and CDC\_DR is obtained from \url{https://www.researchgate.net/publication/363832943_Matlab_code_of_CDC_DR}.

\section{The detailed cluster number estimation results on several data sets}
\label{sec:appendixC}
We use 5 data sets (Ch, Lc, Ly, So, Zo) as examples to compare $\text{Gap}^*$, BIC and $B$ (BKPlot) on varied $k$. The function $\text{BIC}(k)$ and $B(k)$ fail to detect ground-truth cluster number on those data sets, while $\text{Gap}^*(k)$ function can report the right cluster number. From Fig.\ref{fig:estK}, we can observe the difference between the $\text{Gap}^*$ method and two baselines when $k$ is varied from $2$ to $10$.
\begin{figure}[p]
	\includegraphics[scale=0.15]{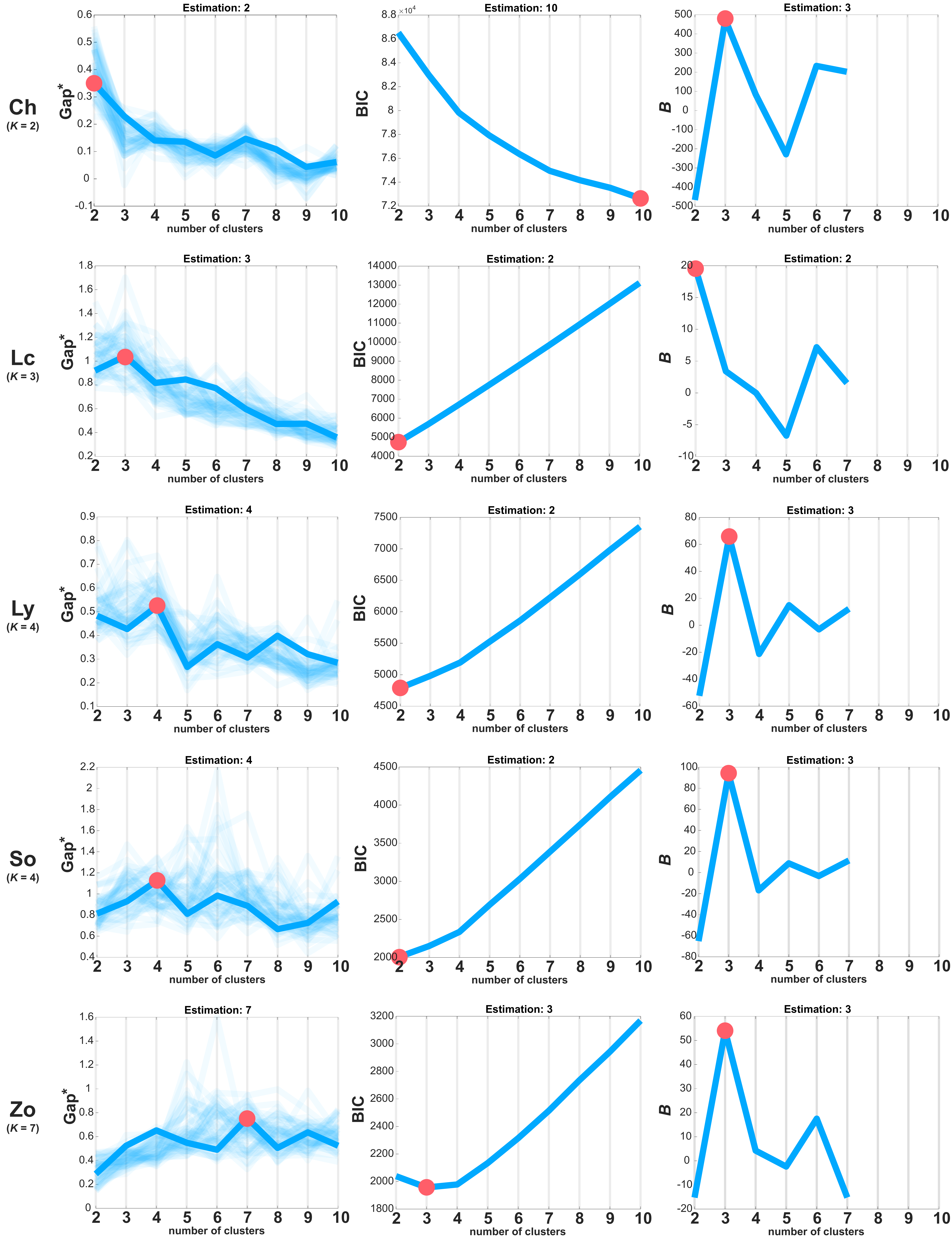}  
	\centering
	\caption{In the line-plots of $\text{Gap}^*(k)$ and $\text{BIC}(k)$, $k$ is varied from $2$ to $10$. Since calculating $B(8)$ requires a value on $k=11$, we only have $B(k)$ values when $k$ is increased from $2$ to $7$. The red dots mark the value of $\hat{k}$ estimated by each method on each data set. In each $\text{Gap}^*$ chart, we have lines corresponding to the $50$ sets of $\text{Gap}^*(k)$ values, where one of the lines for the most frequent $\hat{k}$ is prominently displayed.}
	\label{fig:estK}   
\end{figure}

\section{Extension to numerical data sets}
\label{sec:appendixD}
\textbf{Method}: Firstly, we transform a numerical data set $X$ into a categorical data set $X^c$ by using $k$-means on each attribute, i.e., each $X^c_m$ has $k$ distinct values or labels from the results of $k$-means on $X_m$. Then, we use categorical data clustering methods on $X^c$ to obtain the clustering results. The number of clusters used in all algorithms is set to be the ground-truth cluster number on each data set. \textbf{Results}: We compare the results of $k$-means++ on $X$ and $K$-SigCat, $k$-modes on the corresponding $X^c$. As shown in Table \ref{tab: numerical}, $K$-SigCat is able to integrate the clustering results of $k$-means on each attribute, and achieve better performance than $k$-means. $K$-SigCat can also achieve better performance than $k$-modes on most data sets.
\begin{table}[t]
	\small
	\caption{The performance comparison of $k$-means on $X$, and $K$-SigCat, $k$-modes on the corresponding $X^c$ in terms of ACC, NMI and F-score on 10 UCI data sets $X$. We execute each algorithm 50 times and report its average results.}
	\label{tab: numerical}
	\begin{adjustbox}{width=1\textwidth, center}
		\begin{tabular}{l|ccc|ccc|ccc}
			\toprule
			\multicolumn{1}{l}{} & \multicolumn{3}{c}{ACC} & \multicolumn{3}{c}{NMI} & \multicolumn{3}{c}{F-score} \\
			Data set & $K$-SigCat & $k$-modes & $k$-means & $K$-SigCat & $k$-modes & $k$-means & $K$-SigCat & $k$-modes & $k$-means \\
			\midrule
			iris & \textbf{0.926} & 0.693 & 0.877 & \textbf{0.815} & 0.585 & 0.735 & \textbf{0.871} & 0.715 & 0.809 \\
			ionosphere  & \textbf{0.719} & 0.665 & 0.709 & \textbf{0.154} & 0.063 & 0.125 & \textbf{0.609} & 0.578 & 0.608 \\
			wine & \textbf{0.948} & 0.831 & 0.644 & \textbf{0.818} & 0.611 & 0.414 & \textbf{0.898} & 0.747 & 0.587 \\
			glass & 0.507 & 0.378 & \textbf{0.534} & \textbf{0.373} & 0.166 & 0.361 & 0.406 & 0.296 & \textbf{0.492} \\
			wdbc & \textbf{0.930} & 0.904 & 0.854 & \textbf{0.617} & 0.553 & 0.422 & \textbf{0.878} & 0.846 & 0.788 \\
			movement & 0.374 & 0.312 & \textbf{0.444} & 0.491 & 0.377 & \textbf{0.568} & 0.274 & 0.200 & \textbf{0.349} \\
			vertebral & 0.527 & 0.548 & \textbf{0.596} & 0.243 & 0.260 & \textbf{0.389} & 0.496 & 0.504 & \textbf{0.588} \\
			yeast & 0.209 & 0.201 & \textbf{0.368} & 0.117 & 0.041 & \textbf{0.236} & 0.175 & 0.186 & \textbf{0.282} \\
			leukemia & \textbf{0.942} & 0.436 & 0.713 & \textbf{0.801} & 0.066 & 0.501 & \textbf{0.890} & 0.488 & 0.666 \\
			seeds & 0.863 & 0.865 & \textbf{0.893} & 0.675 & 0.668 & \textbf{0.700} & 0.769 & 0.769 & \textbf{0.809} \\
			\midrule
			Mean & \textbf{0.694} & 0.583 & 0.663 & \textbf{0.510} & 0.339 & 0.445 & \textbf{0.627} & 0.533 & 0.598 \\
			\bottomrule
	\end{tabular}
	\end{adjustbox}
\end{table}

\bibliographystyle{elsarticle-num} 
\bibliography{SCAT-AIJ}

\end{document}